\documentclass[lettersize,conference]{ieeeconf}

\IEEEoverridecommandlockouts                              
                                                          
\usepackage{cite}
\usepackage{amsmath,amssymb,amsfonts}
\usepackage{algorithmic}
\usepackage{graphicx}
\usepackage{textcomp}
\usepackage{xcolor}
\def\BibTeX{{\rm B\kern-.05em{\sc i\kern-.025em b}\kern-.08em
    T\kern-.1667em\lower.7ex\hbox{E}\kern-.125emX}}
    
\usepackage{xspace}
\let\OldTexttrademark\texttrademark
\renewcommand{\texttrademark}{\OldTexttrademark\xspace }%

\usepackage [english]{babel}
\usepackage [autostyle, english = american]{csquotes}
\MakeOuterQuote{"}
    
\usepackage{hyperref}
\newcommand{\email}[1]{\href{mailto:#1}{#1}}

\usepackage{bm}

\newcommand{\overbar}[1]{\mkern 1.5mu\overline{\mkern-1.5mu#1\mkern-1.5mu}\mkern 1.5mu}
\usepackage{cleveref}
\usepackage{afterpage}
\usepackage{gensymb}
\usepackage{paralist}
\newtheorem{theorem}{Theorem}

\renewcommand{\baselinestretch}{1.025} 

\begin{document}


\title{Modeling and Control of Intrinsically Elasticity Coupled Soft-Rigid Robots}

\author{Zach J. Patterson$^{1}$, Cosimo Della Santina$^{2,3}$, and Daniela Rus$^{1}$
\thanks{$^{1}$ Computer Science and Artificial Intelligence Laboratory, MIT. \email{zpatt@mit.edu}, \email{rus@csail.mit.edu}}%
\thanks{$^{2}$ Department of Cognitive Robotics, Delft University of Technology. \email{cosimodellasantina@gmail.com}}%
\thanks{$^{3}$ Institute of Robotics and Mechatronics, German Aerospace Center (DLR).}
}

\maketitle

\begin{abstract}
While much work has been done recently in the realm of model-based control of soft robots and soft-rigid hybrids, most works examine robots that have an inherently serial structure. While these systems have been prevalent in the literature, there is an increasing trend toward designing soft-rigid hybrids with intrinsically coupled elasticity between various degrees of freedom. In this work, we seek to address the issues of modeling and controlling such structures, particularly when underactuated. We introduce several simple models for elastic coupling, typical of those seen in these systems. We then propose a controller that compensates for the elasticity, and we prove its stability with Lyapunov methods without relying on the elastic dominance assumption. This controller is applicable to the general class of underactuated soft robots. After evaluating the controller in simulated cases, we then develop a simple hardware platform to evaluate both the models and the controller. Finally, using the hardware, we demonstrate a novel use case for underactuated, elastically coupled systems in "sensorless" force control. 
\end{abstract}

\section{Introduction}


In recent years, interest has grown in the development of soft-rigid hybrid robots \cite{bernSimulation2022,zhuSoftRigid2023,zhangGeometric2020,coevoetPlanning2022}. Modeling and control of these robots are, for the most part, easily extended from standard methods of soft robot model-based control \cite{dellasantinaModelbasedDynamicFeedback2020a,katzschmannDynamic2019,albu-schafferUnifiedPassivitybasedControl2007} thanks to their straightforward, serial structure and the reasonable approximation of full actuation. However, thanks to our ongoing work in biomimetic robots, we have recently become interested in a special case of soft-rigid hybrids, for which the structure is parallel. This induces intrinsic elastic coupling across degrees of freedom (DOFs), regardless of the choice of state variables or collocation of the actuators. Examples include cases with branches of articulated links embedded in a silicone matrix (see Fig. \ref{fig:example}a), fingers (Fig. \ref{fig:example}b) \cite{raoAnalyzing2017,valero-cuevasTendon2007, armaniniDiscrete2021}, energy storage mechanisms in legged robots \cite{grimmerComparison2012}, and the wings and flippers of flying and swimming animals (Fig. \ref{fig:example}c) \cite{konowSpring2015}. 

We propose to refer to these structures as \textit{intrinsically elastically coupled}. Indeed, elastic couplings are a common feature of reduced order models of soft robots \cite{armanini2023soft}, but it usually arises due to a combination of a selection of spacially intertwined base functions and non-collocation of actuators \cite{pustina2023collocated}.
On the contrary, elastic coupling is intrinsically present in these soft-rigid systems, and it is not artificially introduced by the model order reduction technique.
A first contribution of this work is to make a first step towards mathematically characterizing these systems, and proposing simple yet effective modeling strategies.


\begin{figure}[t]
\centering
\includegraphics[width=0.45\textwidth]{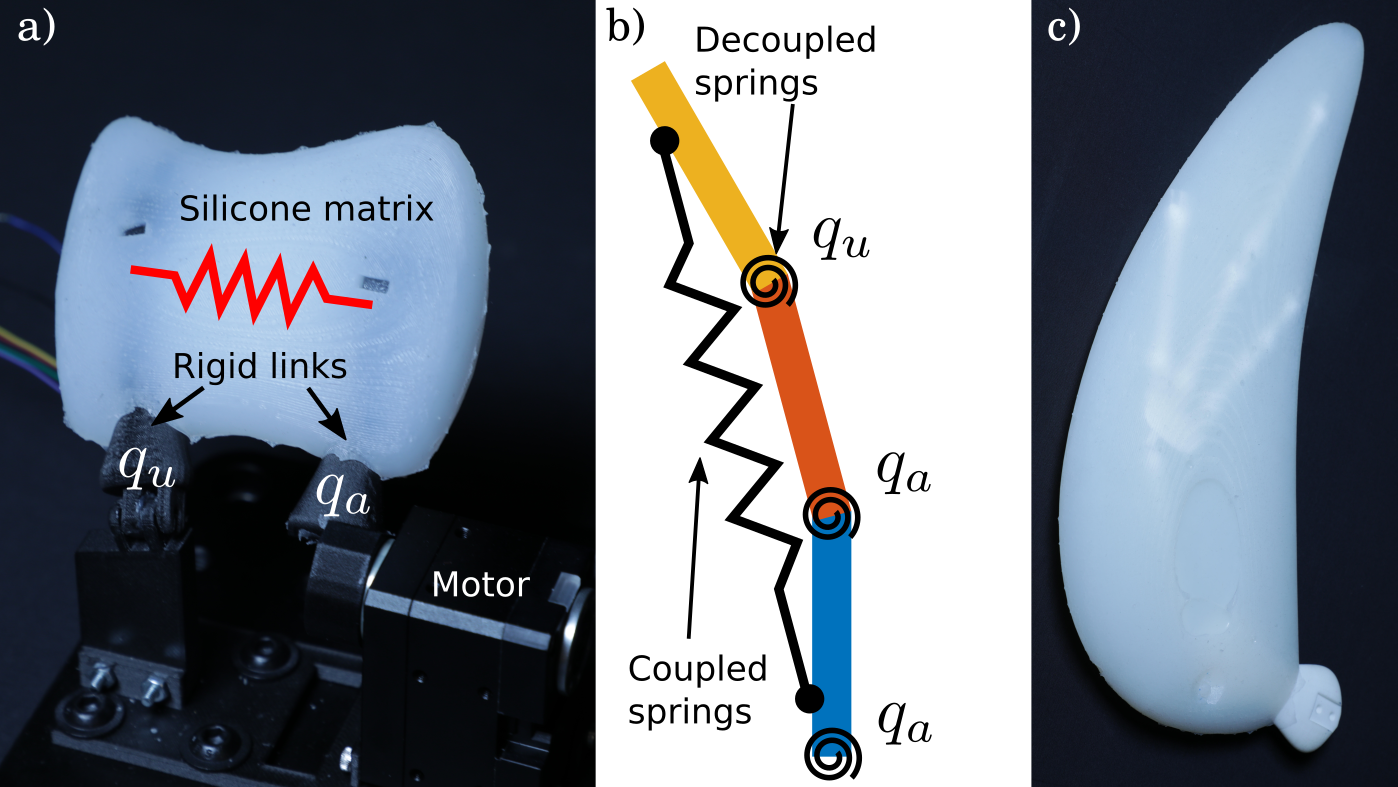}
\caption{Three examples of systems with coupled stiffness. a) Simple hardware implementation inspired by flippers. Two rigid 3D printed links are embedded in a silicone matrix. b) A "finger" with torsional springs on all joints and a coupling between the second and third joint. $q_{\mathrm{a}}$ denotes actuated degrees of freedom (DOF) and $q_{\mathrm{u}}$ denotes unactuated. c) Flipper inspired structure with a rigid "skeleton" embedded in a soft flipper}\label{fig:example}
\end{figure}

Despite their prevalence in nature, parallel rigid-soft structures are often overlooked in soft robotics modeling and control literature. Notable exceptions include \cite{niehuesCompliance2015,raoAnalyzing2017,raoAnalyzing2018}, which focus on control properties in dexterous manipulation. Despite the sparse literature, these systems have been recognized to promise to enhance "embodied intelligence" \cite{iidaTimescales2023}.


Most control works have been done under full actuation hypotheses \cite{della2020model,cao2021model,doroudchi2021configuration,weerakoon2022passivity,shao2023model,zheng2023task}. However, the soft-rigid robots we are concerned with here cannot usually be modeled under this assumption.
%
%
Several works have focused on model-based control of soft robots in their physically accurate underactuated regime. For example, \cite{boyerMacrocontinuous2006} proposes a regulator with good performance but no stability proof, and \cite{wuFEMBased2021,thieffryReduced2018,liEquivalentInputDisturbanceBased2022}, which study the problem in the linear regime. In \cite{dellasantinaModelBasedControlSoft2023}, a PD controller with feedforward stiffness and gravity cancellation is provided, with proven stability under conditions of elastic dominance. This is extended by \cite{borjaEnergybased2022}, confirming stability under the same conditions and introducing a partial gravity compensation mechanism to improve performance. Lastly, \cite{pustinaFeedbackRegulationElastically2022} offers a PD with gravity compensation with stability proof that doesn't rely on elastic dominance but is limited to elastically decoupled systems. In this article, we extend the previous works and propose a controller and proof for underactuated soft robots in general that \textit{do not} rely on elastic dominance for stability. 

We present, therefore, the following contributions:
\begin{compactitem}
    \item Simple models for this new class of intrinsically coupled parallel soft-rigid hybrid robots;
    \item A provably stable PD controller with gravity compensation that can control the collocated configuration variables of these systems\footnote{It is worth noting that this strategy can be seamlessly applied to all kinds of soft robots with or without the presence of (intrinsically) coupled stiffness.};
    \item A sensorless control methodology utilizing the elastic coupling to perform force control;
    \item A simple yet representative of the general challenges testbed;
    \item Simulations and hardware experiments for validation.
\end{compactitem}


\section{Modeling Parallel Elasticity}\label{sec:model}
We propose four models for soft-rigid hybrid robots like Fig. \ref{fig:example}a. See Fig. \ref{fig:models} for graphical depictions. 
\subsection{Linear Coupling}\label{subsec:linear}
In the linear case, coupled joints described by generalized coordinates $q_1$ and $q_2$ have an energy described by 
\begin{equation}
    U_{\mathrm{linear}} = \frac{1}{2} k (q_1 - q_2)^2.
\end{equation}
This results in linear forces $F = k(q_1 - q_2)$ on joint 1 and $F = k(q_2 - q_1)$ on joint 2. Note that the stiffness coupling induces both a coupled and a decoupled component. 

\subsection{Distance-based coupling}
We will now discuss two similar nonlinear models of coupled stiffness. Each parameterizes a segment by an arc-length parameter, $s \in [0,1]$. A point along segment $i$, $p_i(q,s)$, is obtained by solving the forward kinematics. For coupled segments the energy is then modeled by the integral of the distance between the points $p(q,s)$ of each segment:
\begin{equation}
    U_{\mathrm{distance}} = \int_{0}^{1} \|(p_1 - p_2)\|^2 \,ds.
\end{equation}
The resulting forces are then $F = \frac{\partial U}{\partial q}$.

\subsection{Rejection-based coupling}
The energy in this case is described by the magnitude of the rejections along $s$ of $p_i(q,s)$ for each segment. The rejection of $p_1$ on $p_2$ is calculated as 

\begin{equation}
    r_1 = p_1 - \frac{p_1 \cdot p_2}{p_2 \cdot p_2}p_2.
\end{equation}
The rejection-based energy is then taken to be
\begin{equation}
    U_{\mathrm{rej}} = \int_{0}^{1}\|r_1\|^2 + \|r_2\|^2 \,ds.
\end{equation}

\subsection{Neo-Hookean Model}
This candidate model uses the standard continuum mechanics solution for a hyperelastic solid undergoing shear. We simply must calculate the deformation that contributes to shear strain across the coupling material. The strain energy density for a Neo-Hookean solid under shear is 
\begin{equation}
    U_{\mathrm{nh}} = k (\lambda^2 + \lambda^{-2} - 2)
\end{equation}
and the resultant forces are $F = k(\lambda - \lambda^{-3})$, where $\lambda$ is the principle shear stretch calculated for the particular situation.

\begin{figure}[t]
\centering
\includegraphics[width=0.5\textwidth]{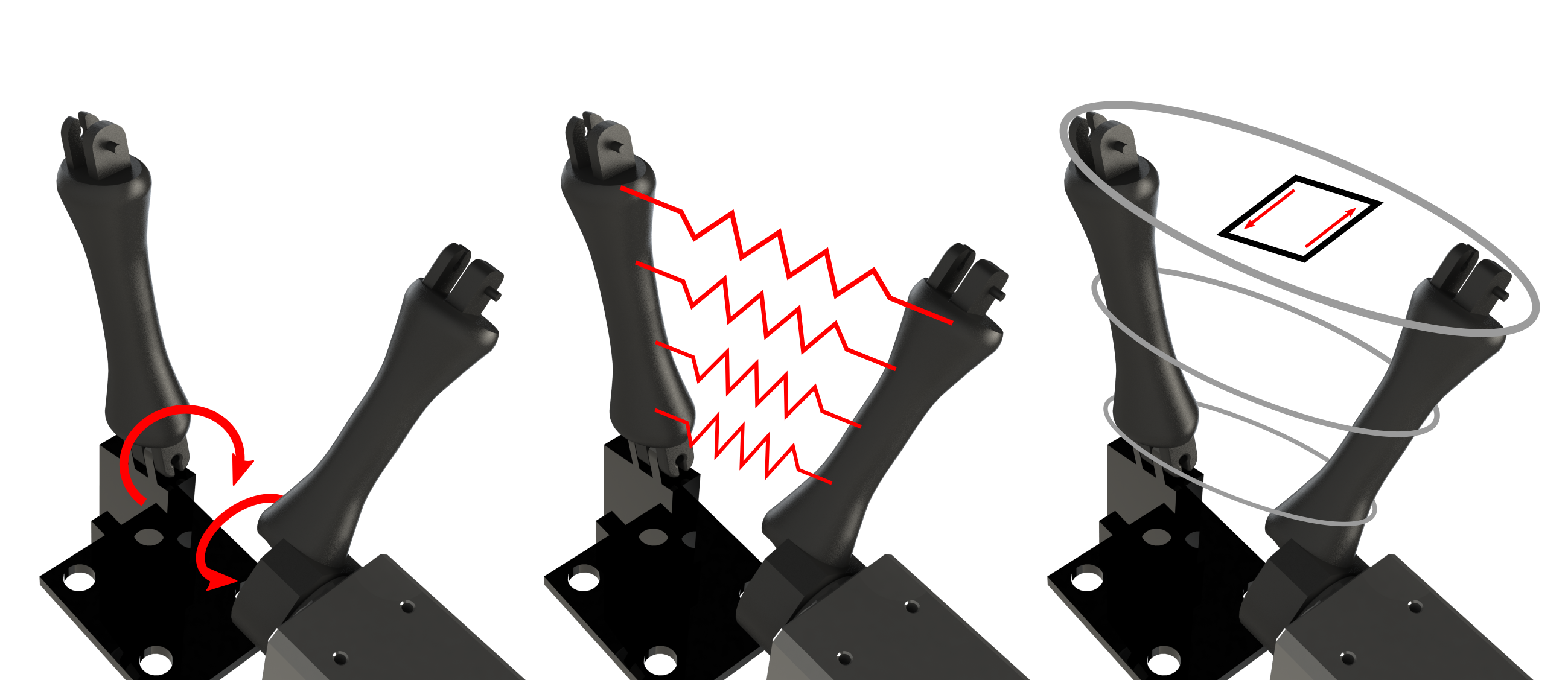}
\caption{Depictions of the qualitative aspects of the models. Left: Linear model. Middle: Distance and Rejection models. Right: Neo-Hookean shear model.}\label{fig:models}
\end{figure}

\subsection{Underactuated Dynamics}
General soft robot dynamics can be modeled using the standard robot manipulator ODEs, with additional terms to account for the elasticity and dissipation:
\begin{equation}\label{eq:dynamics}
    M(q)\ddot{q} + C(q,\dot{q}) + G(q) + F_K(q) + D(q)\dot{q} = A \tau,
\end{equation}
where $q \in \mathbb{R}^n$, $M$ is the robot inertia matrix, $C$ is the Coriolis matrix, $G$ is gravity, $F_K$ is the force from the elastic potential, $D$ is the damping matrix, and $A$ transforms actuator torques $\tau \in \mathbb{R}^m$ into configuration space. 

For an underactuated version of (\ref{eq:dynamics}), we can rewrite our system dynamics in terms of actuated, $q_{\mathrm{a}} \in \mathbb{R}^m$, and unactuated, $q_{\mathrm{u}} \in \mathbb{R}^{n-m}$, state variables as follows

\begin{multline}\label{eq:underactuated}
    \begin{pmatrix}
    M_{\mathrm{aa}} & M_{\mathrm{au}}\\
    M_{\mathrm{ua}} & M_{\mathrm{uu}}
    \end{pmatrix}
    \begin{pmatrix}
        \ddot{q}_{\mathrm{a}} \\
        \ddot{q}_{\mathrm{u}}
    \end{pmatrix} +
    \begin{pmatrix}
    C_{\mathrm{aa}} & C_{\mathrm{au}}\\
    C_{\mathrm{ua}} & C_{\mathrm{uu}}
    \end{pmatrix}
    \begin{pmatrix}
        \dot{q}_{\mathrm{a}} \\
        \dot{q}_{\mathrm{u}}
    \end{pmatrix} +
    \begin{pmatrix}
        G_{\mathrm{a}} \\
        G_{\mathrm{u}}
    \end{pmatrix} \\
    +\begin{pmatrix}
    F_{K,\mathrm{aa}} & F_{K,\mathrm{au}}\\
    F_{K,\mathrm{ua}} & F_{K,\mathrm{uu}}
    \end{pmatrix} + 
    \begin{pmatrix}
    D_{\mathrm{aa}} & D_{\mathrm{au}}\\
    D_{\mathrm{ua}} & D_{\mathrm{uu}}
    \end{pmatrix}
    \begin{pmatrix}
        \dot{q}_{\mathrm{a}} \\
        \dot{q}_{\mathrm{u}}
    \end{pmatrix} =
    \begin{pmatrix}
        A_{\mathrm{a}} \tau \\
        0
    \end{pmatrix}.
\end{multline}
The above expression generalizes the elastically decoupled class of soft robots previously discussed in \cite{pustinaFeedbackRegulationElastically2022} to the case of elastically coupled soft robots, and the coupled case collapses to decoupled when $F_{K,\mathrm{au}} = F_{K,\mathrm{ua}} = 0$. We note that for the case of linear elasticity, 
\begin{equation}
F_K(q) = \begin{pmatrix}
    F_{K,\mathrm{a}}\\
    F_{K,\mathrm{u}}
    \end{pmatrix} = 
    \begin{pmatrix}
    K_{\mathrm{aa}} & K_{\mathrm{au}}\\
    K_{\mathrm{ua}} & K_{\mathrm{uu}}
    \end{pmatrix}
    \begin{pmatrix}
    q_{\mathrm{a}}\\
    q_{\mathrm{u}}
    \end{pmatrix}
    =K q
\end{equation}
\subsection{System Properties}
System (\ref{eq:dynamics}) exhibits several useful features common to robots with dynamics in this form \cite{dellasantinaModelBasedControlSoft2023,murrayMathematicalIntroductionRobotic1994}. First, the inertia matrix is symmetric, positive definite, and bounded in $q$. Second, the Coriolis matrix $C(q,\dot q)$ has the property that there exists a scalar $\gamma_C > 0$ such that $||C(q,\dot q)|| < \gamma_C ||\dot q||$ for any $q,\dot q \in \mathbb{R}^n$. Third, the gravitational potential, force, and Hessian are all bounded. Thus there exist constants $\gamma_{U_G}$, $\gamma_{G}$, $\gamma_{\partial G}$ such that for any $q \in \mathbb{R}^n$,
\begin{equation*}
    ||U_G(q)|| \leq \gamma_{U_G}, \quad ||G(q)|| \leq \gamma_{G}, \quad ||\frac{\partial G(q)}{\partial q}|| \leq \gamma_{\partial G}.
\end{equation*}

\section{Controllers}
\subsection{Zero Dynamics}
For clarity in the coming presentations and without loss of generality, we will use the linear elasticity case. We adopt the notation $\overline q = [ \overline{q}_{\mathrm{a}} \quad q_{\mathrm{u}} ]^T$.

We can examine the zero dynamics of system (\ref{eq:underactuated}) when $q_{\mathrm{a}}$ is forced to $\overline{q}_{\mathrm{a}}$. In this case, $\dot{q}_{\mathrm{a}}, \ddot{q}_{\mathrm{a}} = 0$ and we get
\begin{multline}\label{eq:zero_dynamics}
    M_{\mathrm{uu}}(\overline q) \ddot{q}_{\mathrm{u}} + C_{\mathrm{uu}}(\overline q,0,\dot{q}_{\mathrm{u}}) \dot{q}_{\mathrm{u}} + G_{\mathrm{u}}(q) \\
    + K_{\mathrm{ua}}\overline q_{\mathrm{a}} + K_{\mathrm{uu}} q_{\mathrm{u}}+ D_{\mathrm{uu}} \dot{q}_{\mathrm{u}} = 0.
\end{multline}
As we will show in the following, the dynamics in (\ref{eq:underactuated}) are minimum phase.
\begin{theorem}
    For any initial state and any $q_{\mathrm{a}} = \overline{q}_{\mathrm{a}}$, the trajectories of (\ref{eq:zero_dynamics}) are bounded and converge to $(q_{\mathrm{u}},\dot{q}_{\mathrm{u}}) = (q_{\mathrm{u,eq}},0)$ where $q_{\mathrm{u,eq}}$ is found by solving 
    \begin{equation}
        K_{\mathrm{ua}}\overline q_{\mathrm{a}} + K_{\mathrm{uu}} q_{\mathrm{u}} + G_{\mathrm{u}}(q) = 0.
    \end{equation}
\end{theorem}
\begin{proof}
    We begin by specifying the Lyapunov-like function 
    \begin{equation}
        V = \frac{1}{2}\dot{q}_{\mathrm{u}}^{T}M_{\mathrm{uu}}\dot{q}_{\mathrm{u}} + \overline q^T K \overline q + U_g(q).
    \end{equation}
The gravitational potential, $U_g$, is lower bounded and the elastic potential $U_{K}$ is positive definite, implying that $V$ is lower bounded. Next, we take the time derivative of $V$, using (\ref{eq:zero_dynamics}) to solve for $\ddot{q}_{\mathrm{u}}$:
\begin{align*}
    \dot{V} & =  \dot{q}_{\mathrm{u}}^{T}M_{\mathrm{uu}}\ddot{q}_{\mathrm{u}}\! +\! \frac{1}{2}\dot{q}_{\mathrm{u}}^{T}\dot{M}_{\mathrm{uu}}\dot{q}_{\mathrm{u}}\! +\! \dot{q}_{\mathrm{u}}^{T}(K_{\mathrm{ua}}\overline q_{\mathrm{a}}\! +\! K_{\mathrm{uu}} q_{\mathrm{u}}) + \dot{q}_{\mathrm{u}}^{T}G_{\mathrm{u}}\\
    & = \frac{1}{2}\dot{q}_{\mathrm{u}}^{T}\dot{M}_{\mathrm{uu}}\dot{q}_{\mathrm{u}} + \dot{q}_{\mathrm{u}}^{T}(K_{\mathrm{ua}}\overline q_{\mathrm{a}} + K_{\mathrm{uu}} q_{\mathrm{u}}) + \dot{q}_{\mathrm{u}}^{T}G_{\mathrm{u}} \\
    & \quad + \dot{q}_{\mathrm{u}}^{T}(-C_{\mathrm{uu}}\dot{q}_{\mathrm{u}} - G_{\mathrm{u}} - K_{\mathrm{ua}}\overline q_{\mathrm{a}} - K_{\mathrm{uu}} q_{\mathrm{u}} - D_{\mathrm{uu}}\dot{q}_{\mathrm{u}})\\
    & = \frac{1}{2}\dot{q}_{\mathrm{u}}^{T}(\dot{M}_{\mathrm{uu}} - 2C_{\mathrm{uu}})\dot{q}_{\mathrm{u}} - \dot{q}_{\mathrm{u}}^T D_{\mathrm{uu}}\dot{q}_{\mathrm{u}}\\
    & = - \dot{q}_{\mathrm{u}}^T D_{\mathrm{uu}}\dot{q}_{\mathrm{u}} \leq 0.
\end{align*}
Because $V$ is radially unbounded and lower bounded, this shows that the trajectories of the unactuated states are bounded. By Lasalle's invariance principle \cite{hassan2002nonlinear} (relying on the Corollary to Lasalle discussed in the Appendix of \cite{pustinaFeedbackRegulationElastically2022}), since $\dot{V}$ is negative definite, the trajectories converge and the the proof is complete. 
\end{proof}
\subsection{PD Regulator}
Here we extend the PD-style regulator with online gravity cancellation to underactuated systems with coupled stiffness. The controller is proved to be asymptotically stable. The controller and especially the proof are inspired by the work of Pustina et. al \cite{pustinaFeedbackRegulationElastically2022}. The collocated controller is
\begin{equation}\label{eq:controller}
    \tau = G_{\mathrm{a}}(q) - K_{\mathrm{D}} \dot q_{\mathrm{a}} + K_{\mathrm{au}} \overline q_{u} + K_{\mathrm{aa}} \overline{q}_{\mathrm{a}} + K_{\mathrm{P}} (\overline{q}_{\mathrm{a}} - q_{\mathrm{a}}),
\end{equation}
where $\overline q_{u}$ is an arbitrary $q_{\mathrm{u}} \in \mathbb R^{n-m}$. For the following, we drop dependencies where they are clear from context.

\begin{theorem}

There exists a $K_{\mathrm{P}}$ such that if $\overbar q_{\mathrm{u}} = q_{\mathrm{u,eq}}$ then the trajectories of the closed loop system (\ref{eq:underactuated}) are bounded and converge asymptotically to $(q_{\mathrm{a}}, q_{\mathrm{u}}, \dot q_{\mathrm{a}}, \dot q_{\mathrm{u}})=(\overline{q}_{\mathrm{a}}, q_{\mathrm{u,eq}}, 0, 0)$, where $q_{\mathrm{u,eq}}$ is the solution of $G_{\mathrm{u}} + K_{\mathrm{ua}} \overline q_{a} + K_{\mathrm{uu}} \overline{q}_{\mathrm{u}} = 0$.
\end{theorem}

\begin{proof}
Consider the Lyapunov-like function 
\begin{equation}
\begin{split}
    V(\widetilde{q},\dot q) &= \gamma_1 \left( \frac{1}{2}\dot q^T M(q) \dot q + \frac{1}{2}\widetilde{q}^T \hat K \widetilde{q} \right. \\
    &\left. - \frac{e_{\mathrm{a}}^T G_{\mathrm{a}}(q)}{1 + 2 e_{\mathrm{a}}^T e_{\mathrm{a}}} - \frac{e_{\mathrm{a}}^T K_{\mathrm{au}} \overline q_{u}}{{1 + 2 e_{\mathrm{a}}^T e_{\mathrm{a}}}} + U_G(q) \right) \\
    &+ \frac{2e_{\mathrm{a}}^T (M_{\mathrm{aa}}(q) \dot q_{\mathrm{a}} + M_{\mathrm{au}}(q) \dot q_{\mathrm{u}})}{1 + 2 e_{\mathrm{a}}^T e_{\mathrm{a}}},
\end{split}
\end{equation}
where $\gamma_1 > 0$ and 
\begin{equation*}
    \hat{K} = \begin{pmatrix}
    K_{\mathrm{P}} + F_{K,aa} & F_{K,au}\\
    F_{K,ua} & F_{K,uu}
    \end{pmatrix},
    \hat{D} = \begin{pmatrix}
    K_{\mathrm{D}} + D_{\mathrm{aa}} & D_{\mathrm{au}}\\
    D_{\mathrm{ua}} & D_{\mathrm{uu}}
    \end{pmatrix}.
\end{equation*}
We first show that V is lower bounded. We know from the properties of system (\ref{eq:dynamics}) that the following terms are lower bounded as $\frac{2e_{\mathrm{a}}^T (M_{\mathrm{aa}}(q) \dot q_{\mathrm{a}} + M_{\mathrm{au}}(q) \dot q_{\mathrm{u}}}{1 + 2 e_{\mathrm{a}}^T e_{\mathrm{a}}} \geq - \lambda_{\mathrm{max}}(M) ||\dot q|| $, $- \frac{e_{\mathrm{a}}^T G_{\mathrm{a}}(q)}{1 + 2 e_{\mathrm{a}}^T e_{\mathrm{a}}} \geq -\alpha_G$, $U_G(q) \geq - \alpha_{U_G}$, and $\frac{1}{2}\widetilde{q}^T \hat K \widetilde{q} - \frac{e_{\mathrm{a}}^T K_{\mathrm{au}} \overline q_{u}}{{1 + 2 e_{\mathrm{a}}^T e_{\mathrm{a}}}} > 0$. We obtain
\begin{equation}\label{eq:bnds}
    V(\widetilde{q},\dot q) \frac{\gamma_1}{2} \lambda_{\mathrm{min}} (M) ||\dot q||^2 - 2 \lambda_{\mathrm{max}}(M)||\dot q||  - \gamma_1 (\alpha_G + \alpha_{U_G}).
\end{equation}
The function on the right is quadratic and has a known minimum (see \cite{pustinaFeedbackRegulationElastically2022}) with value $\gamma_2$. Therefore, $V(\widetilde{q},\dot q) \geq \gamma_2 > -\infty$ and so is lower bounded. From (\ref{eq:bnds}), $V(\widetilde{q},\dot q)$ is also radially unbounded and thus its conditions to be a Lyapunov-like function are fulfilled. \\
We now consider the time derivative of $V(\widetilde{q},\dot q)$:
\begin{equation}
\begin{split}
    &\dot V(\widetilde{q},\dot q) = \gamma_1 \left( \frac{1}{2}\dot q^T \dot M \dot q + \dot q^T M \ddot q + \dot q^T (F_K + G)+ \dot{q}^T \hat K \widetilde{q} \right. \\
    &\left. - \frac{e_{\mathrm{a}}^T \frac{\partial}{\partial q}( G_{\mathrm{a}} + K_{\mathrm{au}} \overline q_{u})\dot q}{1 + 2 e_{\mathrm{a}}^T e_{\mathrm{a}}} + \frac{4 e_{\mathrm{a}}^T (G_{\mathrm{a}} + K_{\mathrm{au}} \overline q_{u})\dot q^T e_{\mathrm{a}}}{(1 + 2 e_{\mathrm{a}}^T e_{\mathrm{a}})^2} \right) \\
    & + \frac{2\dot q_{\mathrm{a}}^T (M_{\mathrm{aa}} \dot q_{\mathrm{a}} + M_{\mathrm{au}} \dot q_{\mathrm{u}})}{1 + 2 e_{\mathrm{a}}^T e_{\mathrm{a}}} + \frac{2 e_{\mathrm{a}}^T (M_{\mathrm{aa}} \ddot q_{\mathrm{a}} + M_{\mathrm{au}} \ddot q_{\mathrm{u}})}{1 + 2 e_{\mathrm{a}}^T e_{\mathrm{a}}}\\
    &  + \frac{2e_{\mathrm{a}}^T (\dot M_{\mathrm{aa}} \dot q_{\mathrm{a}} + \dot M_{\mathrm{au}} \dot q_{\mathrm{u}})}{1 + 2 e_{\mathrm{a}}^T e_{\mathrm{a}}} - \frac{8e_{\mathrm{a}}^T (M_{\mathrm{aa}} \dot q_{\mathrm{a}} + M_{\mathrm{au}} \dot q_{\mathrm{u}})e_{\mathrm{a}}^T \dot q_{\mathrm{a}}}{(1 + 2 e_{\mathrm{a}}^T e_{\mathrm{a}})^2}.
\end{split}
\end{equation}
We substitute the closed loop dynamics for action (\ref{eq:controller}) on system (\ref{eq:underactuated}) and simplify algebraically to obtain
\begin{equation}
\begin{split}
    &\dot V(\widetilde{q},\dot q) = \gamma_1 \left( -\dot q^T \hat D \dot q  + \frac{2 e_{\mathrm{a}}^T e_{\mathrm{a}} \dot q^T (G_{\mathrm{a}} + K_{\mathrm{au}} \overline q_{u})}{(1 + 2 e_{\mathrm{a}}^T e_{\mathrm{a}})}  \right. \\
    &\quad \left. - \frac{e_{\mathrm{a}}^T \frac{\partial G_{\mathrm{a}}}{\partial q}\dot q}{1 + 2 e_{\mathrm{a}}^T e_{\mathrm{a}}} + \frac{4 e_{\mathrm{a}}^T (G_{\mathrm{a}} + K_{\mathrm{au}} \overline q_{u})\dot q^T e_{\mathrm{a}}}{(1 + 2 e_{\mathrm{a}}^T e_{\mathrm{a}})^2} \right) \\
    &\quad  + \frac{2\dot q_{\mathrm{a}}^T (M_{\mathrm{aa}} \dot q_{\mathrm{a}} + M_{\mathrm{au}} \dot q_{\mathrm{u}})}{1 + 2 e_{\mathrm{a}}^T e_{\mathrm{a}}}  \\
    &\quad - \frac{2 e_{\mathrm{a}}^T \hat D_{\mathrm{a}} \dot q}{1 + 2 e_{\mathrm{a}}^T e_{\mathrm{a}}} + \frac{2e_{\mathrm{a}}^T (\dot q_{\mathrm{a}}^T \dot C_{\mathrm{aa}} + \dot q_{\mathrm{u}}^T\dot C_{\mathrm{ua}})e_{\mathrm{a}}}{1 + 2 e_{\mathrm{a}}^T e_{\mathrm{a}}} \\
    & \quad - \frac{2 e_{\mathrm{a}}^T \hat K_{\mathrm{a}} e_{\mathrm{a}}}{1 + 2 e_{\mathrm{a}}^T e_{\mathrm{a}}} - \frac{8e_{\mathrm{a}}^T (M_{\mathrm{aa}} \dot q_{\mathrm{a}} + M_{\mathrm{au}} \dot q_{\mathrm{u}})e_{\mathrm{a}}^T \dot q_{\mathrm{a}}}{(1 + 2 e_{\mathrm{a}}^T e_{\mathrm{a}})^2}.
\end{split}
\end{equation}
The first and ninth terms (involving $\hat D$ and $\hat K$) are negative definite. The other terms can be upper bounded by positive functions, as noted in \cite{pustinaFeedbackRegulationElastically2022}, as follows
\begin{equation}\label{eq:bound}
    \dot V(\widetilde{q},\dot q) \leq - 
    \begin{pmatrix}
        ||\dot q|| \\ \cfrac{|| e_{\mathrm{a}} ||}{\sqrt{1 + 2 ||e_{\mathrm{a}}||^2}} 
    \end{pmatrix}^T 
    Q
    \begin{pmatrix}
        ||\dot q|| \\ \cfrac{|| e_{\mathrm{a}} ||}{\sqrt{1 + 2 ||e_{\mathrm{a}}||^2}} 
    \end{pmatrix},
\end{equation}
with the nonzero entries of matrix Q equal to
\begin{gather*}
    Q_{11} = \gamma_1 \lambda_{\textrm{min}}(\hat D) - \frac{\gamma_c}{\sqrt{2}} - 4\lambda_{\textrm{max}}(M) \\
    Q_{12} = Q_{21} = -(\gamma_1 \alpha_{GK} + \sigma_{\textrm{max}}(\hat D_{\mathrm{a}}))\\
    Q_{22} = 2\lambda_{\textrm{min}}(K_{\mathrm{P}} + K_{\mathrm{aa}})\\
\end{gather*}
where $\alpha_{GK} = 2\alpha_G + \alpha_{\partial G} + 2K_{\mathrm{au}} ||\overline q_{u}||$.
Therefore, $\dot V \leq 0$ for $Q > 0$. Using the Sylvester criteria, $Q$ will be positive definite if and only if
\begin{gather}
    \gamma_1 \lambda_{\textrm{min}}(\hat D) - \frac{\gamma_c}{\sqrt{2}} - 4\lambda_{\textrm{max}}(M) > 0,\label{eq:cond1}\\
\begin{split}
    \textrm{det}\,Q &= 2\lambda_{\textrm{min}}(K_{\mathrm{P}} + K_{\mathrm{aa}}) \Big(\gamma_1 \lambda_{\textrm{min}}(\hat D) - \frac{\gamma_c}{\sqrt{2}} - 4\lambda_{\textrm{max}}(M)\Big) \\
    & -  \Big(2\gamma_1\alpha_{GK} + \sigma_{\textrm{max}}(\hat D_{\mathrm{a}})\Big)^2 > 0.
\end{split} \label{eq:cond2}
\end{gather}
These conditions are met for the following $\gamma_1$ and $K_{\mathrm{P}}$
\begin{gather}
    \gamma_1 > \cfrac{\gamma_C + 4 \sqrt{2}\lambda_{\textrm{max}}(M)}{\sqrt{2}\lambda_{\textrm{min}}(\hat D)} \label{eq:gamma}\\
    \lambda_{\textrm{min}}(K_{\mathrm{P}} + K_{\mathrm{aa}}) > \cfrac{\Big(2 \gamma_1 \alpha_{GK} + \sigma_{\textrm{max}}(\hat D_{\mathrm{a}})\Big)^2}{2\Big(\gamma_1 \lambda_{\textrm{min}}(\hat D) - \frac{\gamma_c}{\sqrt{2}} - 4\lambda_{\textrm{max}}(M)\Big)}.\label{eq:kp}
\end{gather}
This last term is verified hypothesis. Thus, combining (\ref{eq:cond1}), (\ref{eq:cond2}), (\ref{eq:gamma}), (\ref{eq:kp}) shows that $\dot V \leq 0$, allowing the application of LaSalle's Principle \cite{hassan2002nonlinear}. Also, from (\ref{eq:bound}), we can see that for $Q > 0$, $\dot V = 0$ if and only if $e_{\mathrm{a}} = 0$ and $\dot q = 0$. Therefore, the trajectories of the closed loop system converge, proving the thesis.
\end{proof} 

Note that this proof can be extended to the case with feedback from $q_{\mathrm{u}}$ in the coupling term if we allow the lower bound on $K_{\mathrm{P}} + K_{\mathrm{aa}}$ to increase with $||q_{\mathrm{u}}||^2$.


The control law (\ref{eq:controller}) is enough to provide global convergence to a unique equilibrium under certain conditions on the stiffness. For the sake of space, we refrain from writing out the proof, but the condition for convergence is dominance of the elasticity over gravity (e.g. $K > - \frac{\partial^2 U_G(q)}{\partial q}$). See \cite{dellasantinaModelBasedControlSoft2023} for details.

\section{Simulations}
We test our regulator on the model system presented in Fig. \ref{fig:example}b. For simplicity, we implement the linear stiffness model from Section \ref{subsec:linear}. The parameters for the model are as follows. Each link has length $l = 1m$ and mass $m = 0.1kg$. Gravity is $g = 9.81 m/s^2$, the decoupled stiffness of each joint is $k_{\mathrm{d}} = 1.5 Nm$, the coupled stiffness between the second and third joints is $k_{\mathrm c} = 2 Nm$, the damping is $d = 0.5$, and the control gains are $K_{\mathrm{P}} = 1$ and $K_{\mathrm{D}} = 0.5$. We seek to track a reference signal $q_{\mathrm{a}} = [-1.1 \quad 0.7]^T$. Trajectories and torques are shown in Fig. \ref{fig:sim}. Our controller is able to successfully track the reference, while the standard PD controller fails to eliminate steady state error in the collocated, coupled DOF $q_{\mathrm{a},2}$.
\begin{figure}[t]
\centering
\includegraphics[width=0.5\textwidth]{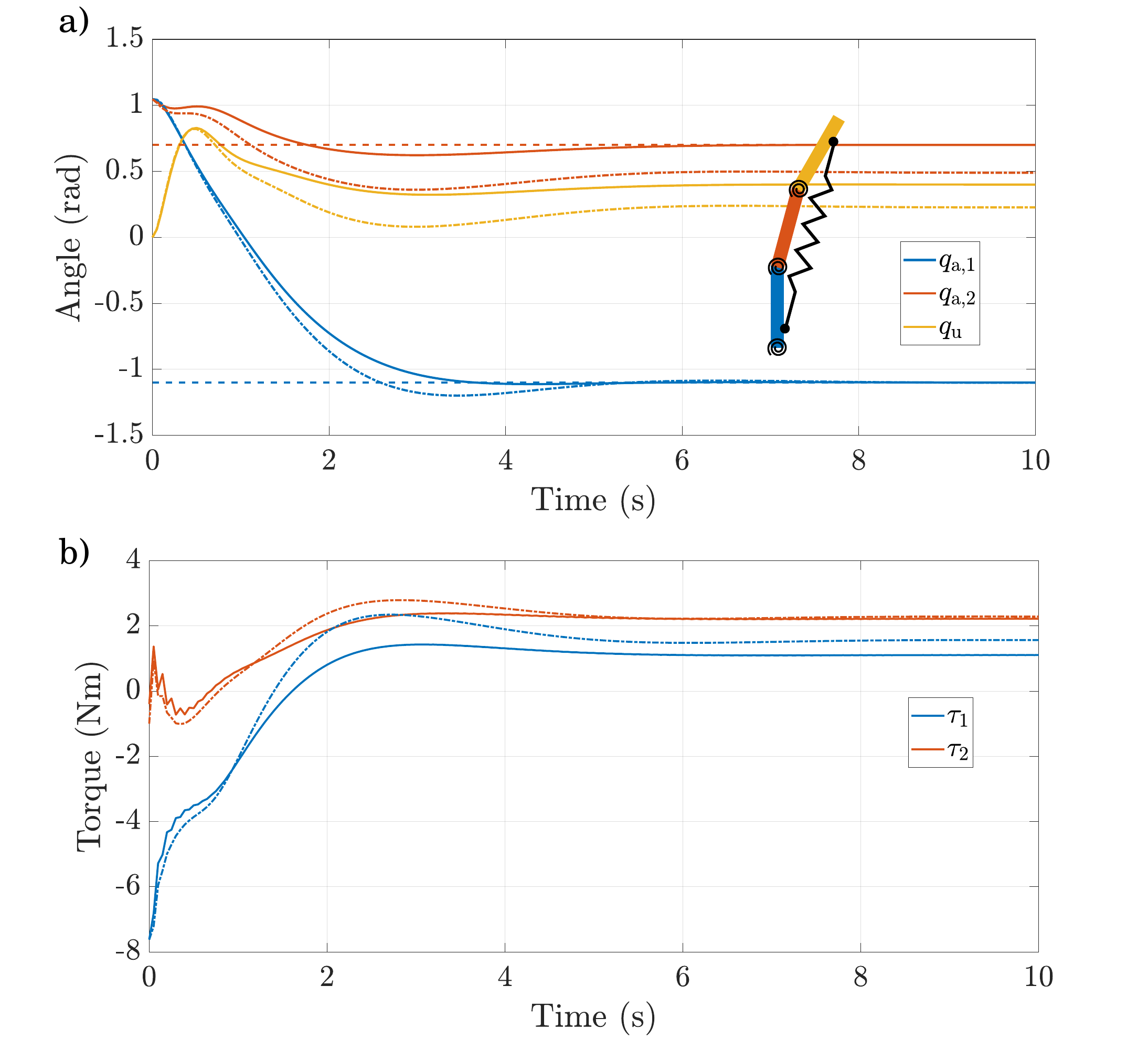}
\caption{A simulated demo on a simplified finger model under gravity. Solid lines designate our regulator whereas dash-dot lines designate an identical PD regulator without the coupling compensation. a) Shows the trajectories with an inset graphic of the model. b) Shows the control torques.}\label{fig:sim}
\end{figure}

We also simulate the "flipper" system shown in Fig. \ref{fig:example}c. We perform both a regulation and a trajectory tracking task with a single actuated coupled DOF and compare our controllers with the standard soft robot PD regulator. We observe in Fig. \ref{fig:flipper} that our regulators are able to successfully achieve set point regulation and their performance on trajectory tracking is reasonable. The naive approach performs much worse as it does not account for coupled elasticity. 

\begin{figure}[t]
\centering
\includegraphics[width=0.5\textwidth]{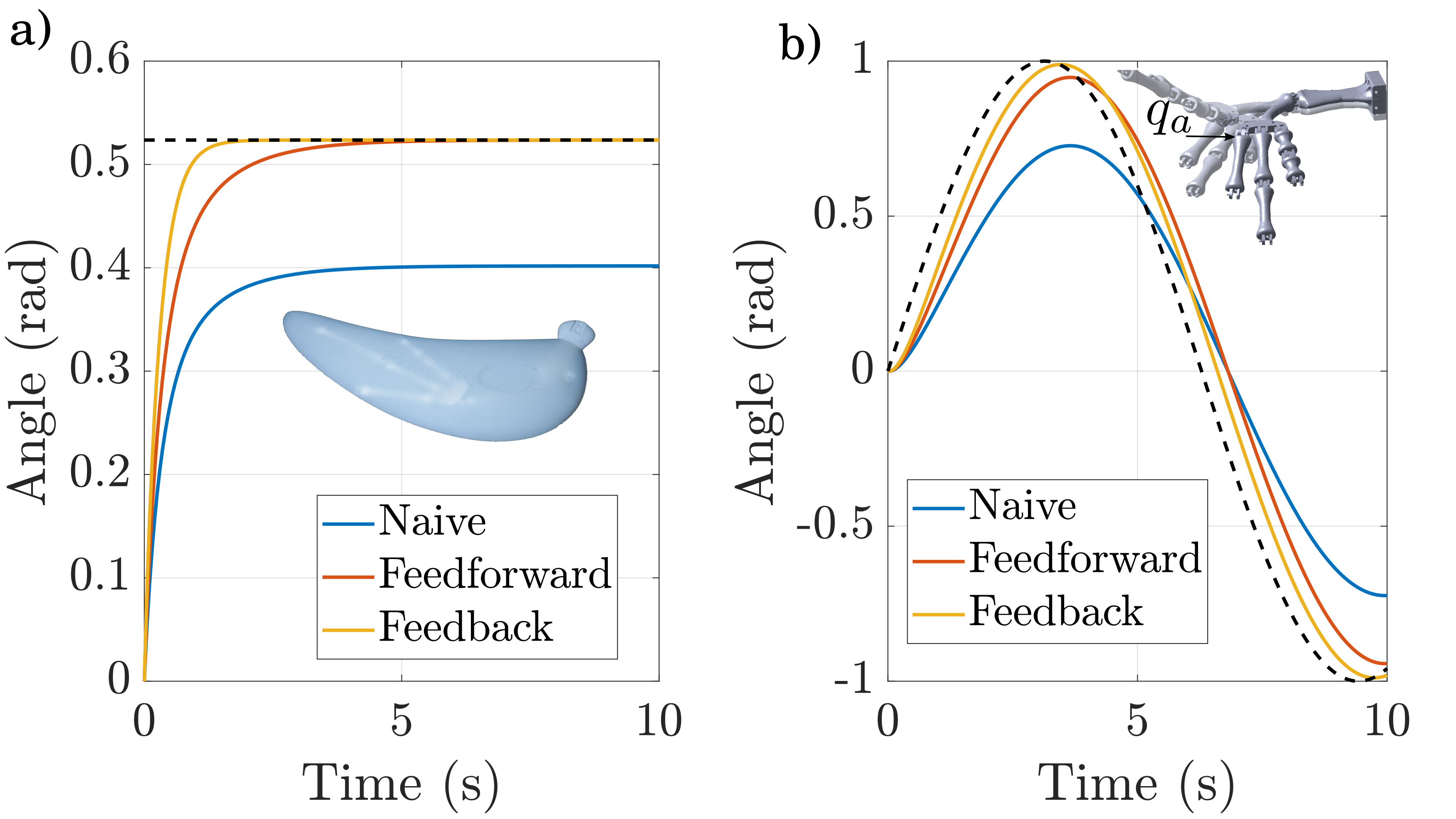}
\caption{A simulated demo of our controllers on a simplified "flipper" model. a) Shows the trajectories with an inset image of the real-life version of the idealized system. b) Shows snapshots taken from the simulated model along the trajectory.}\label{fig:flipper}
\end{figure}

\section{Experiments}
We produced a simple robot to validate our above models and regulators, and to explore the potential benefits of using such coupled elastic systems. The robot can be seen in Fig. \ref{fig:example}b. Briefly, the design is as follows. We 3D print two links that are embedded in an elastic matrix via silicone injection molding \cite{bellInjectionMoldingSoft2022}. One link is attached to a Dynamixel XM430-W350-T servo motor and the other contains a magnetic encoder (design from \cite{liuModularBioinspiredRobotic2023}) and is unactuated. The links are attached to a common fixed base.

\subsection{Model Validation}
To test the models from Section \ref{sec:model}, we perform a series of experiments wherein the motor is commanded to go to a specific state and the free link is either free to move or is held at a fixed location. States and inputs are collected to calculate the coupling force based on each of the models. Because each of the coupling models is linear in the stiffness parameter $k$, we can perform a simple ordinary least squares (OLS) analysis for identification of that parameter. Based on OLS we can also easily get a notion of the goodness of fit of the model based on the ratio of explained variance ($R^2$). For our dataset ($N=90$), the Neo-Hookean model performed the best with $R^2 = 0.9$ followed by the linear model with $R^2 = 0.89$, the distance-based coupling with $R^2 = 0.89$ and the rejection-based coupling with $R^2 = 0.75$.

\begin{figure}[t]
\centering
\includegraphics[width=0.5\textwidth]{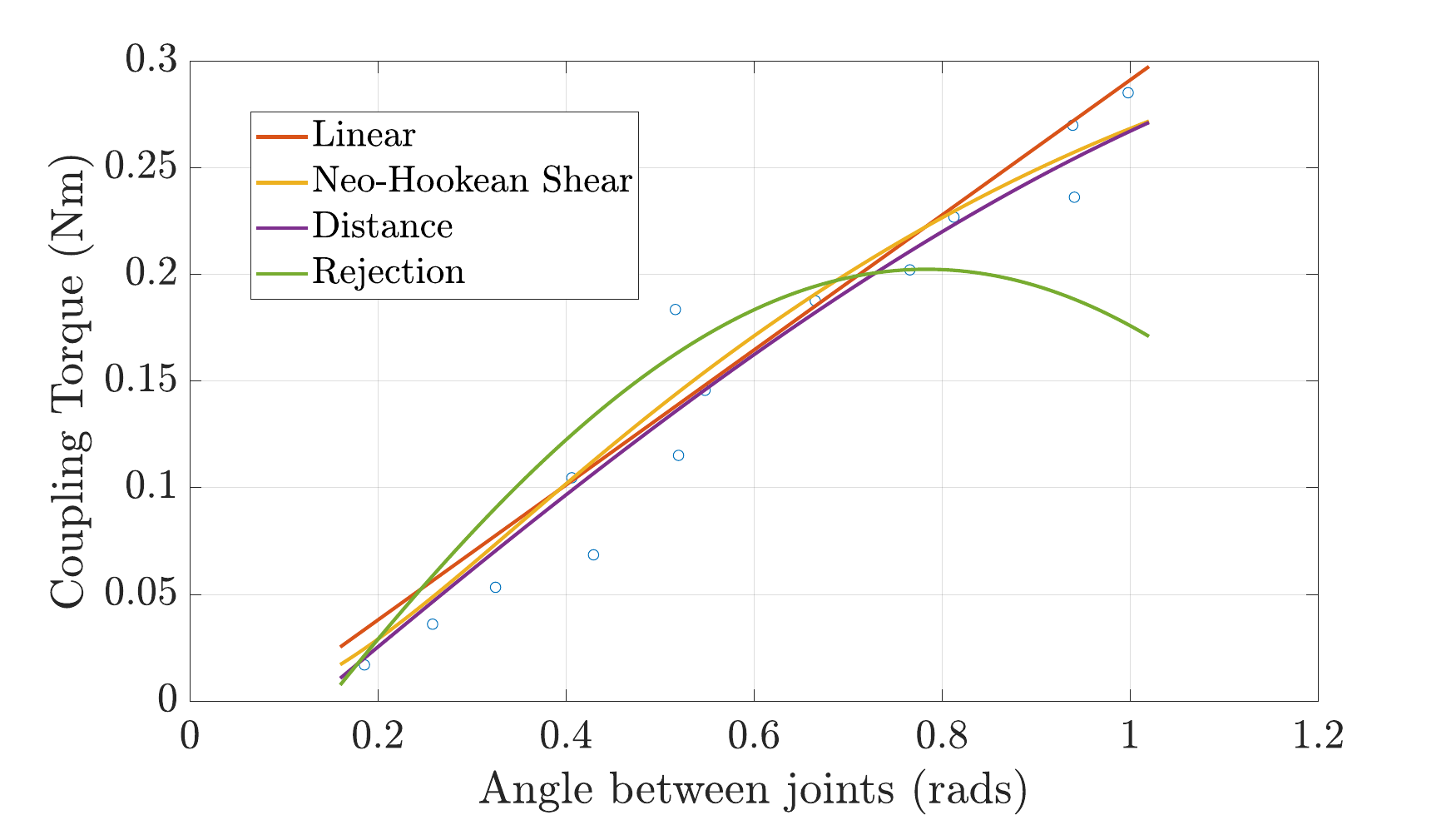}
\caption{Models comparison on a subset of the data.}\label{fig:model_comp}
\end{figure}

\subsection{Coupled Force Control}
Coupled elastic systems such as those shown in Fig. \ref{fig:example} instantiate a form of embodied intelligence, similar to series elastic actuators, that allows the estimation of forces based on the state of the system and a model of the elastic coupling. To test this concept, we use the hardware shown in Fig. \ref{fig:example}b. For all hardware experiments, we neglected gravity given the low mass of the elements of the robot. Given the good performance of the linear coupling model, we use it for simplicity. We implement a simple PID control loop over the approximate force at the end of link 2 as estimated from the elastic coupling. We show that a PID controller with coupling compensation is able to converge to steady state torques with relatively low error, even without direct feedback or force sensing. The controller used is \begin{multline}
    \tau = K_{\mathrm{P}}(F_k(q) - F_{\mathrm d}) - K_{\mathrm{D}} \dot q \\+ K_{\mathrm I} \int_{0}^{t} (F_k(q) - F_{\mathrm d})\,dt + F_k(\overbar q) + G_{\mathrm{a}}(q),
\end{multline}
where $F_k(q)$ is the estimate of the current force on the end of the link based on the coupling torque, $F_{\mathrm d}$ is the desired force, $e_{\mathrm f} = F_k(q) - F_{\mathrm d}$, and $F_k(\overbar q)$ is the same compensation that was discussed in our collocated regulator. For our particular implementation, we neglect gravity ($G_{\mathrm{a}}(q) = \overbar 0$).To test this controller, we actuate our robot until it makes contact with a Robotous RFT60-HA01 force-torque sensor, which we use to collect ground truth. After calibration, results are shown in Fig. \ref{fig:force}. We observe from Fig. \ref{fig:force}b that the errors are relatively low; under $15\%$. Stills from a force control demo are shown in Fig. \ref{fig:snaps}.
\begin{figure}[t]
\centering
\includegraphics[width=0.45\textwidth]{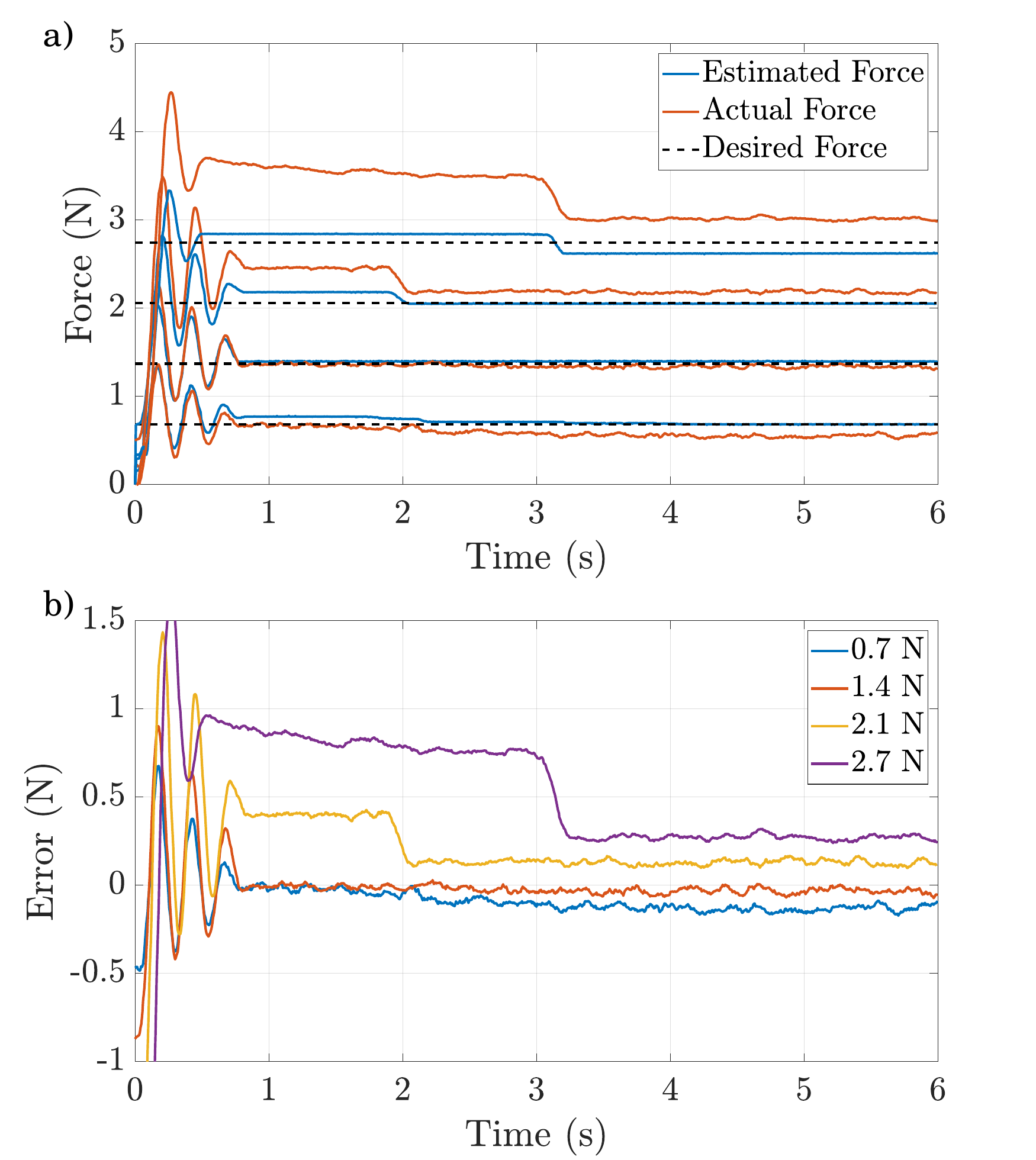}
\caption{Results of sensorless force control on our hardware. a) Estimated force (based on the coupling model), actual force (measured by the force-torque sensor), and desired force for four different desired forces. b) Errors for the force control experiments. }\label{fig:force}
\end{figure}

\begin{figure*}[t]
\centering
\includegraphics[width=0.95\textwidth]{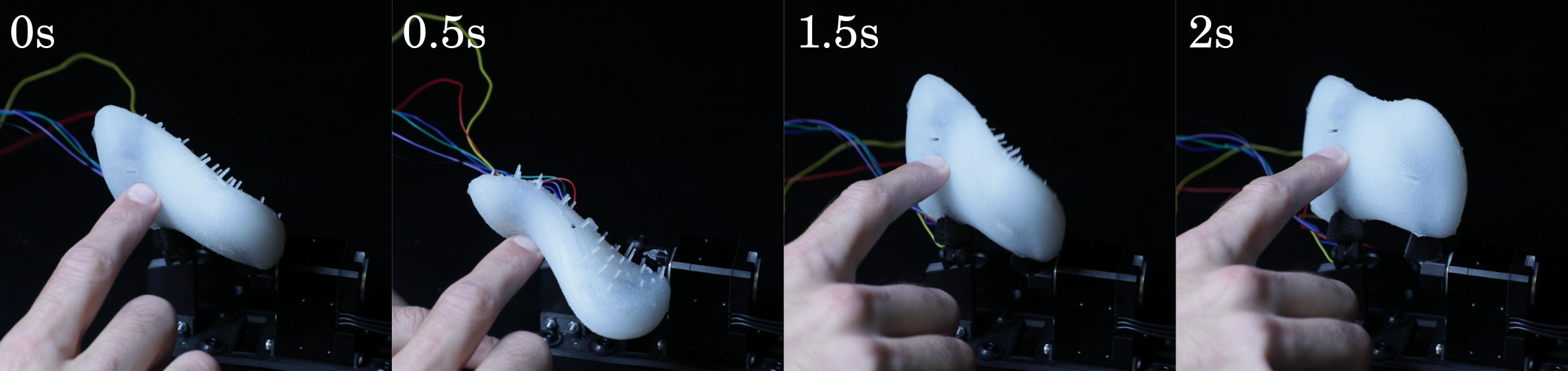}
\caption{Stills from a force control experiment. The controller is set to exert 2.7N of force on the finger.}\label{fig:snaps}
\end{figure*}

\subsection{Regulation}
A set point was commanded and the state of the robot observed. The unactuated joint of the robot was disturbed during the trials. Trajectories and torques for controllers with feedforward compensation, feedback compensation, and without coupling compensation are shown in Fig. \ref{fig:regulator}. Both compensation cases are able to successfully reach the set point. However, when disturbed, the case with feedforward compensation can still be moved off the set point by large disturbances in the unactuated DOF, whereas feedback compensation fully rejects these disturbances. 

\begin{figure*}[t]
\centering
\includegraphics[width=0.95\textwidth]{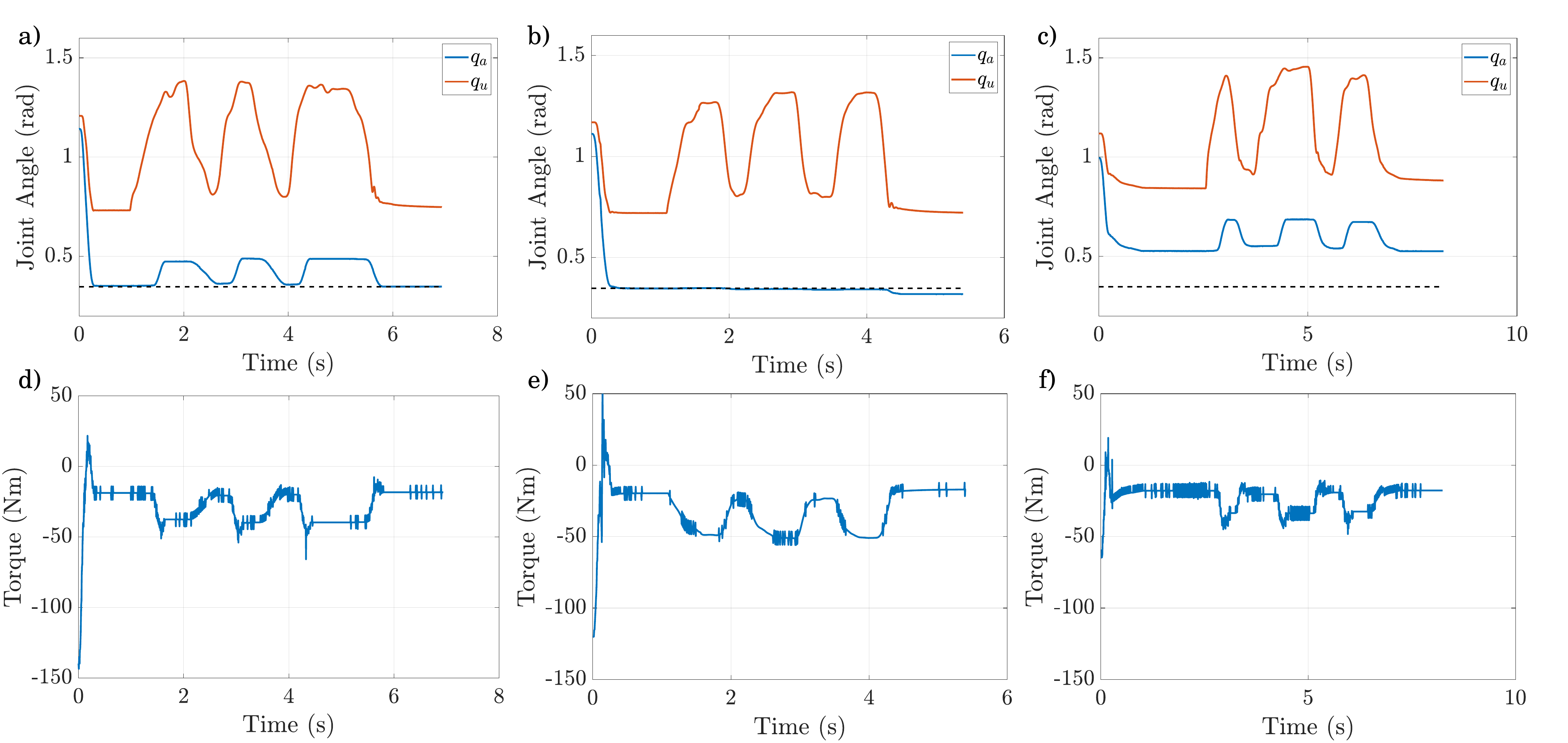}
\caption{Results on hardware with coupled stiffness. Spikes in activity correspond to perturbations on the unactuated joint. Trajectories a) with feedforward compensation, b) with feedback compensation, c) without compensation. Torques d) with feedforward compensation, e) with feedback compensation, and f) without compensation.}\label{fig:regulator}
\end{figure*}

\section{Discussion}
The models presented in this work proved effective at modeling the coupling of the real system. Interestingly, the linear model proved very effective. In fact, given that the Neo-Hookean shear model was only marginally more effective, from the control engineer's perspective the linear model seems to be the best choice for its simplicity. Further analysis using more complex continuum mechanics and/or FEA along with experiments on other hardware systems are necessary to explain this finding. We note that the rejection-based model performed relatively poorly compared to the others. As seen in Fig. \ref{fig:model_comp}, this is due to poor performance at high angular differences, which is in turn due to the periodicity of the model itself, inherited from the geometry of the rejection. Indeed, when restricting the model to within this threshold, it performs similarly to the others. As is, this model is not a good candidate due to this restrictiveness. 

One of the most exciting contributions in this work is the force control implementation using the coupling model as a "sensor." This opens up a new avenue of research in terms of using elastically coupled degrees of freedom for force control in applications where accuracy is not needed (e.g. in this work we found errors up to $15\%$). This approach is inspired by and similar in spirit to series elastic actuators with intrinsic force sensing \cite{prattSeries1995}. This finding also demonstrates a new instance of the "embodied intelligence" principle \cite{iidaTimescales2023}, in which the intelligence of the mechanical structure is leveraged to enable new functionality. In future work, we will operationalize this approach in more complex systems.

A core limitation of this work is that the regulator does not verify exponential convergence (see the tracking trials in Fig. \ref{fig:flipper}b). Indeed, an open problem in the soft robot control literature is the lack of tracking controllers for general underactuated soft robots \cite{dellasantinaModelBasedControlSoft2023}. This will be a topic of future work.

In conclusion, we presented several simple modeling candidates for coupled elastic systems and evaluated them on harware. We also proposed a new control law for underactuated, coupled elastic systems and proved stability. The control law was evaluated in simulation on several systems as well as on hardware. Finally, we demonstrate a novel use case for coupled elastic systems as force sensors.

\section*{Acknowledgment}
We would like to thank Pietro Pustina for his beautiful proof and for his generous feedback in adapting it to the topic of this work. This work was done with the support of National Science Foundation EFRI program under grant number 1830901 and the Gwangju Institute of Science and Technology.

\clearpage

\bibliographystyle{ieeetran}
\bibliography{ICRA2024}

\end{document}